\documentclass[letterpaper,twocolumn,10pt]{article}
\usepackage{usenix,epsfig,endnotes}

\usepackage{times}
\usepackage{amsmath}
\usepackage{hyperref}
\usepackage{url}
\usepackage{amsthm}
\usepackage{algorithmic}
\usepackage{algorithm}
\usepackage{amssymb}
\usepackage{amsfonts}
\usepackage{graphicx}
\usepackage{epsfig}
\usepackage{subfigure}
\usepackage{lipsum}
\usepackage{footnote}
\usepackage{enumitem}
\usepackage{array}

\usepackage{balance}
\usepackage{flushend}

\newtheorem{theorem}{Theorem}%[section]

\newtheorem{proposition}[theorem]{Proposition}

\begin{document}

\title{Reliable Agglomerative Clustering}

\author{Morteza Haghir Chehreghani\\
\normalsize Department of Computer Science and Engineering \\
\normalsize Chalmers University of Technology \\
\normalsize Email: {morteza.chehreghani@chalmers.se}
}


% compsoc conferences compared to traditional conferences), use this
%Telephone: (800) 555--1212, Fax: (888) 555--1212}


\maketitle

%\footnote{Morteza Haghir Chehreghani is with Chalmers University of Technology, Department of Computer Science and Engineering.
%Mostafa Haghir Chehreghani is with Amirkabir University of Technology, Department of Computer Engineering and IT.
%}

\begin{abstract}
Standard agglomerative clustering suggests  establishing  a new \emph{reliable} linkage at every step. However, in order to provide adaptive, density-consistent and flexible solutions, we study  extracting all the \emph{reliable} linkages at each step, instead of the smallest one. Such a strategy can be applied with all common criteria for agglomerative hierarchical clustering. We also  study that  this strategy with the \emph{single} linkage criterion yields a minimum spanning tree algorithm. We perform experiments on several real-world datasets to demonstrate the  performance of this  strategy compared to the standard alternative.
\end{abstract}

\section{Introduction}

Clustering plays an essential role in data processing and management such as text processing, image segmentation, compression, summarization, knowledge management, network analysis, and bioinformatics. The goal of data clustering is to partition the data into groups such that the objects in the same cluster are more similar in some sense, compared to the inter-cluster objects.
A category of clustering methods partition the data into $K$ flat clusters via for example optimizing a cost/objective function. Examples of this type of methods are $K$-means~\cite{mcqueen1967smc}, normalized cut~\cite{Shi:2000:NCI} and spectral clustering~\cite{Shi:2000:NCI,Ng01onspectral}, where all produce \emph{flat} clusters without any explicit relation between them. In practice, however, the different clusters often do not carry the same information content, i.e., some are more detailed than the others. Thus, in an exploratory data analysis approach, it is desired to propose the clusters at different levels and resolutions, such that both general and specific information are preserved. In this way, the user has more control to choose the desired resolution or even investigate the clusters at different levels and resolutions. For this reason, hierarchical clustering is often more practical is many applications and situations, where the results are usually presented by a dendrogram. A dendrogram is a tree wherein each node represents a cluster and its final nodes (the nodes connected to only one other node) correspond to the objects. A node at a higher level includes the combination of the lower-level clusters and the edge weights (and their lengths) represent the inter-cluster distances.

Hierarchical clustering methods, in general, fall into two categories: agglomerative (bottom-up) and divisive (top-down) \cite{Maimon:2005}.
Agglomerative algorithms consider each object as a separate cluster, and then combine the clusters in a greedy manner to build larger clusters, until at the end there is only one single cluster.
Divisive methods, in an opposite way, start with a single cluster including all objects. Then, at each step, the clusters are divided into two parts to produce finer clusters.
Agglomerative methods are more common for hierarchical clustering, and they are usually computationally more efficient than divisive methods~\cite{podani2000introduction}. In these approaches, the clusters might be combined or divided according to different criteria, e.g., \emph{single}, \emph{complete}, \emph{average}, \emph{centroid} and \emph{Ward}.

Several methods have been developed to improve the different aspects of these algorithms.
\cite{AckermanB16} studies the locality and outer consistency of agglomerative algorithms in an axiomatic way.
The works in \cite{biom12647,Levenstien2003} consider the statistical significance of hierarchical clustering. \cite{Dasgupta:2016,Charikar:2017:AHC,Roy:2016:HCV,Roy:2017:HCV,NIPS2017_7200} investigate the optimization aspects of hierarchical clustering and develop several approximate solutions. To provide robustness in pairwise inter-clusters relations, K-Linkage in \cite{Yildirim2017} investigates multiple pairs of distances for each pair of clusters, \cite{Balcan:2014:RHC} uses  global information for determining the similarities between the clusters,
 \cite{ChehreghaniCA12} trains a Bayesian network to infer the relations between the items to be clustered,
and \cite{ChehreghaniAC08} suggests applying agglomerative methods to small dense subsets of the data instead of the original data. The work in \cite{CHIDANANDAGOWDA1978105} performs the hierarchical clustering on $K$-nearest neighbor graph where fixing a proper $K$ (and the other hyper-parameters) can be nontrivial as discussed in \cite{Luxburg07}. The works in \cite{KarypisHK99,FrantiVH06} might suffer from the same issues.
The methods in \cite{doi:10.1198,jcgs2012} investigate combining aggolomerative methods with probabilistic models which then yields an extra computational complexity. Finally, \cite{Gu:2013:EBC,abs-1109-2378,Cochez:2015:TTA} develop efficient (and approximate) implementations of aggolomerative methods.

%Axiomatic hierarchical clustering : Hierarchical Clustering: Objective Functions and Algorithms: aximatic

In this paper, we focus on agglomerative hierarchical clustering. We consider that the standard  agglomerative algorithms usually select a minimal \emph{reliable} linkage at each step. We call a linkage between two clusters \emph{reliable} if both clusters are the nearest neighbors of each other.  Linkages represent the inter-cluster distances according to a criterion such as  \emph{single} or \emph{average} distance. A \emph{reliable} linkage provides the two clusters at its two sides to be consistent and share similar properties. However, in order to be adaptive w.r.t. the data diversity and variability, we investigate extracting at each step all the \emph{reliable} linkages, instead of the smallest one. This strategy, called \emph{reliable agglomerative clustering}, enables every object to potentially contribute from the early steps of constructing the dendrogram and, thus, clusters with different shapes and densities can evolve from the beginning. This  strategy, similar to the standard agglomerative procedure, can be used with all the common criteria, and it is adaptive to the shape and density of the clusters. A similar idea has been proposed in \cite{Bruynooghe1977} in an abstract form  without further investigations and analysis.
We note that this contribution is orthogonal to the aforementioned methods which aim to improve in particular agglomerative clustering, such that  any of those improvements can be employed with this strategy  too. For example, similar to \cite{ChehreghaniAC08}, we may build the dendrogram from the dense subsets of the data or use global information for computing the base pairwise (dis)similarities \cite{Balcan:2014:RHC}. We may also apply the feature extraction method in \cite{ChehreghaniC20} to infer proper unsupervised representations.
 In the following, inspired by the equivalence of  \emph{single} linkage clustering and the Kruskal's algorithm for computing minimum spanning trees~\cite{Kruskal1956}, we study that \emph{reliable agglomerative clustering} with \emph{single} criterion also yields a minimum spanning tree.  We perform extensive experiments on several real-world datasets to demonstrate the performance of this method compared to the standard approach.

%A rather relevant method is sometimes used in clustering for constructing neighborhood graphs, called mutual $k$-nearest neighbor graph~\cite{Brito1997,MaierHL09}. In this setting, each object first selects  its $k$ nearest neighbors. Then, a valid edge is established if and only if the objects at its two sides choose each other. Thus, the \emph{flat} (i.e., not hierarchical) clusters can be obtained by enumerating the connected components~\cite{Brito1997}. However, there are several fundamental differences between this method and our approach: The mutual $k$-nearest neighbor method establishes the edges between the objects and not the clusters. Then, most of the established edges occur inside small dense regions or small clusters. For this reason, as investigated in~\cite{Luxburg07,MaierHL09}, the selected $k$ must be a large number, i.e., in the order of the number of objects $n$. However,  such a choice can easily lead to combining even well-separated clusters.  In particular, when the size of clusters varies or the clusters have elongated shapes, then, even a fairly small $k$ can combine different clusters. Thereby, as suggested in~\cite{Luxburg07,MaierHL09}, mutual $k$-nearest neighbor is  appropriate only for detecting significant cluster seeds. On the other hand, this method yields only flat clusters.
%
The rest of the paper is organized as the following. In the second section we introduce the reliable agglomerative  strategy, and in the third section we study the connection to minimum spanning trees. We experimentally investigate \emph{reliable agglomerative clustering} in the next section, and finally, we conclude the paper in the last section.

\section{Reliable Agglomerative Clustering}

\begin{figure*}[!t]
    \centering
     \subfigure[]
    {
        \includegraphics[width=0.25\textwidth]{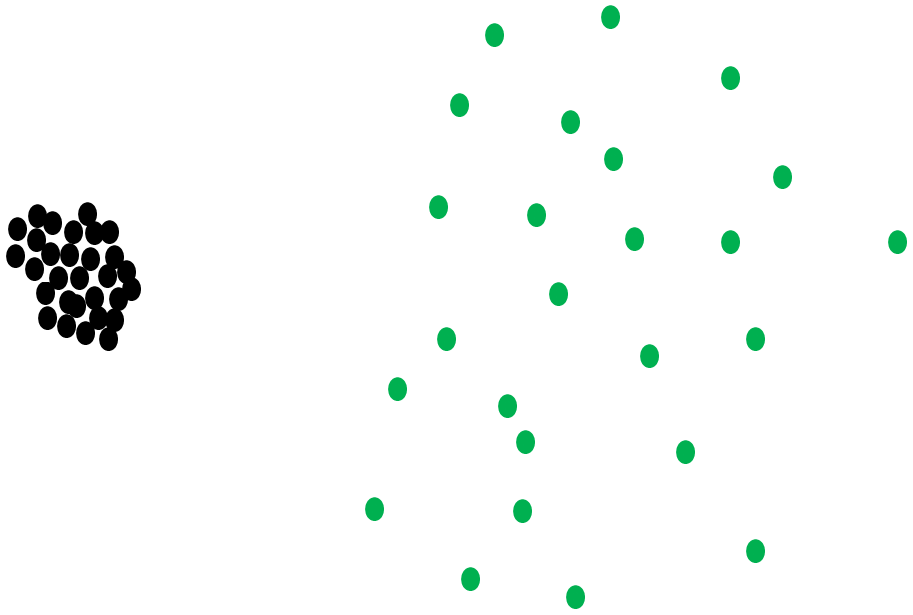}
        \label{fig:Context-Sensitive-Clust1}
    }
    \unskip\ \vrule\
    \subfigure[]
    {
        \includegraphics[width=0.15\textwidth]{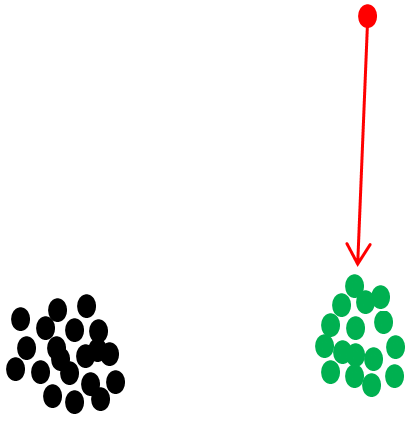}
        \label{fig:Context-Sensitive-Clust2}
    }
    \unskip\ \vrule\
    \subfigure[]
    {
        \includegraphics[width=0.18\textwidth]{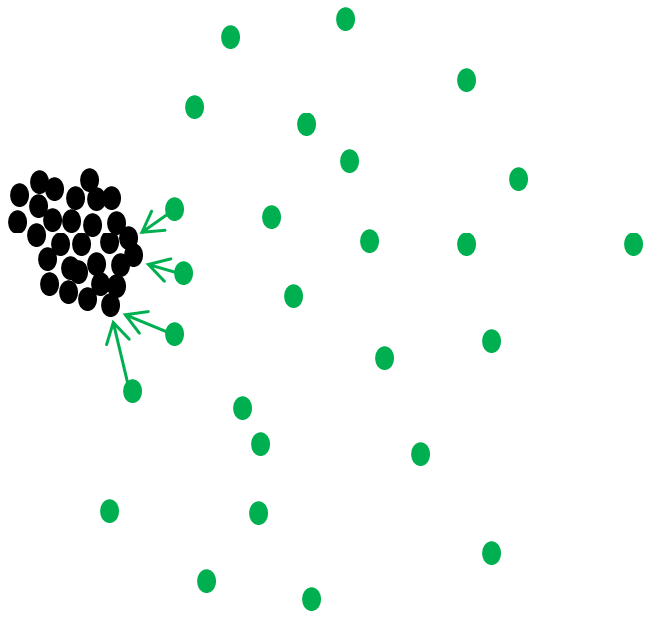}
        \label{fig:Context-Sensitive-Clust3}
    }
    \unskip\ \vrule\
    \subfigure[]
    {
        \includegraphics[width=0.15\textwidth]{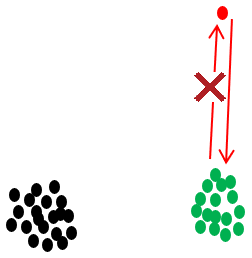}
        \label{fig:Context-Sensitive-Clust4}
    }
    \caption{The standard agglomerative clustering is sensitive to data diversity (Fig \ref{fig:Context-Sensitive-Clust1}). On the other hand, allowing each object/cluster to connect to its nearest neighbor (regardless of if it happens in the nearest neighborhood of the other side too), can lead to inappropriate results (Fig \ref{fig:Context-Sensitive-Clust2} and \ref{fig:Context-Sensitive-Clust3}). Therefore, extracting only \emph{reliable} linkages
    (but all instead of the smallest)
     avoids such situations (Fig\ref{fig:Context-Sensitive-Clust4}).}
    \label{fig:Context-Sensitive-Clust}
\end{figure*}

In this section, we describe \emph{reliable agglomerative clustering} and discuss the connection to computing a minimum spanning tree.

\subsection{A generic view to  agglomerative clustering}
Data are characterized by a set of $n$ objects $\mathbf O= \{0,...,n-1\}$ and a relevant representation. The representation can be for example the vectors in a vector space  or the pairwise dissimilarities between the objects. In the former case, the measurements are shown by the $n\times d$ matrix $\mathbf X$, where the $i^{th}$ row (i.e., $\mathbf X_i$) corresponds to the $d$ dimensional vector of the  $i^{th}$ object. In the latter form, an $n \times n$ matrix $\mathbf D$ represents the pairwise dissimilarities between the objects. A cluster  is shown by $C_p$, which is the set of the object indices that it contains. The function $dist(C_p,C_q)$ denotes the inter-cluster distances that can be defined according to different criteria.

Agglomerative methods follow an iterative procedure where at each step, two clusters (nodes) are combined to build a larger cluster. The procedure continues until there is only one cluster left. The algorithm at each step selects the two clusters that have a \emph{minimal} distance according to a criterion, i.e., a specific definition of $dist(.,.)$. For example, the \emph{single} linkage criterion~\cite{sneath1957dn09j} defines the distance between two clusters as the distance between the nearest members of the clusters. Opposite to this strategy, the \emph{complete} linkage criterion~\cite{lance67hierarchical} defines the distance of two clusters as the distance between their farthest members, that corresponds to the maximum within-cluster distance of the new cluster. On the other hand, in \emph{average} criterion~\cite{sokal58} the average of inter-cluster distances is used as the distance between the two clusters.
Some other methods, e.g., the \emph{centroid} and the \emph{median} criteria, determine a representative for each cluster and then compute the inter-cluster distances by the distances between the representatives. For example, with the \emph{centroid} criterion the representatives are the means of the clusters and at each step, the two clusters with closest centroids are combined to construct a larger cluster.

Another category of agglomerative methods aim to optimize a criterion such as \emph{homogeneity}.  An important instance is the \emph{Ward} method~\cite{Inchoate:Ward63} which aims to minimize  the total within-cluster variance at each step.
However, this criterion can be written as
\begin{align}
dist(C_p,C_q) &= \sum_{i\in C_p\cup C_q}||\mathbf X_i -  \mathbf m_{C_p\cup C_q}||^2  \nonumber \\
& \;\; - \sum_{i\in C_p}||\mathbf X_i -  \mathbf m_{C_p}||^2 - \sum_{j\in C_q}||\mathbf X_j -  \mathbf m_{C_q}||^2   \nonumber \\
&= \frac{|C_p||C_q|}{|C_p|+|C_q|} ||\mathbf m_{C_p} -  \mathbf m_{C_q}||^2 \, , \nonumber
\end{align}
where $ \mathbf m_{C_p}$ denotes the centroid  vector of  cluster $C_p$.

Thus, the \emph{Ward} method also  at each step combines the two clusters with a minimal distance, where the inter-cluster distances are defined as the distances between the cluster means normalized by a function of the size of the clusters.

\subsection{Reliable agglomerative clustering strategy}
We begin with analyzing  the performance of the \emph{single} linkage method, in particular on the data with diverse densities. Such an analysis can be applied to the other criteria as well. We  first consider the data shown in Figure~\ref{fig:Context-Sensitive-Clust1}, which includes two clusters with different densities. The \emph{single} linkage method starts first from the dense data cloud at the left side (shown by black points) and then performs grouping the members of the cluster at the right side (shown by green points). Such that if we stop the clustering early, then, we will have only the members of the cluster at the left side grouped together. The reason is that picking the smallest inter-cluster distance (linkage) does not necessarily yield contributing every object/cluster to building the dendrogram. In particular, as we saw, this approach is sensitive to the density of the clusters and tends to first extract the densest clusters. One way to overcome this issue and take the variance of clusters into account is to require each object/cluster to participate in building the dendrogram. One might interpret the standard agglomerative strategy for selecting the smallest inter-cluster linkage  as i) find the nearest neighbors of the current objects/clusters to obtain the set of potential linkages\footnote{The nearest neighbors are defined according to the $dist(.,.)$ function, which can encode any criterion (e.g., \emph{single}, \emph{complete}, \emph{average}, \emph{centroid} and \emph{Ward}).}, and ii) then pick the smallest linkage.

Therefore, one way to render contributing many objects/clusters in building the dendrogram is to choose all the linkages instead of the smallest one, which makes the dendrogram grow simultaneously from all the objects.
However, allowing all the linkages corresponding to any nearest neighbor might be inappropriate, as it can be sensitive to the presence of outliers or to the clusters which are close but have different densities. Two examples are illustrated in Figures~\ref{fig:Context-Sensitive-Clust2} and~\ref{fig:Context-Sensitive-Clust3}. If we pick  all of the linkages, the red object at the top in Figure~\ref{fig:Context-Sensitive-Clust2}  would establish a linkage to the green cluster (with the closest object of it) at the first level of the dendrogram. However, we know that such a linkage should be established at a higher level, after the members of the green data cloud merge and build their own cluster first. Therefore, this linkage is not a \emph{reliable} linkage, as the two objects at its two sides do not share similar properties and densities. On the other hand, in Figure~\ref{fig:Context-Sensitive-Clust3}, the two green and black clusters are close to each other, such that some objects of the green data cloud choose the members of the black data cloud as  the nearest neighbors, instead of choosing from the green data cloud. This occurs due to the different  densities of the clusters. Therefore, one should be careful in choosing any nearest neighbor linkage. In these examples, the objects/clusters at the two sides of a linkage have different properties and densities. In the example of Figures~\ref{fig:Context-Sensitive-Clust2}, the red object is an outlier whose neighborhood is empty, unlike the neighborhood of the object at the other side, which is significantly denser. Thus, the red object establishes a linkage with one of the green objects, but this object selects another object as its nearest neighbor. In Figures~\ref{fig:Context-Sensitive-Clust3}, some of the green objects establish linkage to some of the black objects, which have a different (i.e., higher) densities around. Therefore, the black objects do not select these green objects as their nearest neighbors. This analysis leads to investigate the \emph{reliability} of linkages established by different objects/clusters, defined in {Definition 1}.

\noindent\textbf{Definition 1}.  \emph{A linkage between two clusters $C_p$ and $C_q$ is `reliable' if and only if both clusters are nearest neighbors of each other, i.e., $C_q \in nn(C_p)$ and $C_p \in nn(C_q)$, where $nn(C_p)$ returns the nearest clusters of cluster $C_p$.}\footnote{A cluster may include only one single object, i.e., each object is a cluster at the lowest level of the dendrogram.}

Note that a cluster might have several nearest neighbors, i.e., $|nn(C_p)|\ge 1$.

Therefore, instead of establishing the linkage(s) from every cluster/object, we select only a subset that are \emph{reliable}. Such an approach provides the clusters at the two sides of a linkage to share consistent neighborhood and densities. Thus, merging them to build a larger cluster becomes meaningful. Then, it avoids non-robust linkages, for example merging the outlier objects at the lowest levels (Figure \ref{fig:Context-Sensitive-Clust4}).
Proposition \ref{lemma:stable_linkage_smallest} indicates that  a linkage with a minimal length is \emph{reliable}, i.e., the standard  agglomerative strategy which combines only the nearest clusters at each step performs reliable selections.

\begin{proposition}
Given a set of clusters $\{C_i\}$ and the respective linkages between them, a linkage with minimal length (called $e^*$) is a `reliable' linkage.
%\footnote{The proofs are discarded, due to space limit.}
\label{lemma:stable_linkage_smallest}
\end{proposition}
\begin{proof}[Proof sketch]
We denote the clusters at the two sides of $e^*$ respectively $C_p$ and $C_q$. Since $e^*$ has a minimal length among all linkages, thus, it will also be the smallest linkage connected to $C_p$ and the same for $C_q$. Therefore, $C_p$ is the nearest neighbor of $C_q$ and $C_q$ is the nearest neighbor of $C_p$,  which makes the corresponding linkage (i.e., $e^*$) \emph{reliable}.
\end{proof}

However, a minimal linkage is \emph{not} the only \emph{reliable} linkage, in particular when the data contain clusters with diverse densities, as demonstrated in Figure \ref{fig:Context-Sensitive-Clust1}. Thus, in order to build the dendrogram in a density-aware and adaptive way, at each level we may select \emph{all} the linkages that are \emph{reliable}. Thereby, this strategy at each step first finds \emph{all the reliable} linkages, and then combines the respective clusters to build a larger cluster at a higher level. Algorithm \ref{alg:strong_agglomerative} describes the procedure in detail  providing an implementation of the high-level method in  \cite{Bruynooghe1977}.

\begin{algorithm}[ht!]
\caption{Reliable Agglomerative Clustering.
%, consistent with the high-level method in \cite{Bruynooghe1977}.
}
\label{alg:strong_agglomerative}
\begin{algorithmic} [1]
\REQUIRE {Objects and the measurements.}
\ENSURE List of the clusters at different levels stored in $Cluster\_List$.
%\vspace{2mm}

\STATE $l =0\;\;\;\;\;\;$ // $l$ specifies the current level
\FORALL{$i \in \mathbf O$}
\STATE $Cluster\_List[l].add(\{i\})$
\ENDFOR

%\vspace{2mm}

\WHILE{$|Cluster\_List[l]| > 1$}

%\vspace{2mm}

\STATE  $min\_dist = []$

\FOR {$0 \le p < |Cluster\_List[l]|$}
	\STATE $min\_dist.add(\min_{0 \le q < |Cluster\_List[l]| , q \ne p}$\\$\;\,dist(Cluster\_List[l][p],Cluster\_List[l][q]))$
\ENDFOR

%\vspace{2mm}

\STATE Initialize matrix $\mathcal G$ by $\mathbf 0$.
% $\mathcal G \in \{0,1\}^{|Cluster\_List[level]| \times |Cluster\_List[level]|}$ by $\mathbf 0$.

\FOR {$p,q \in \{0,..,|Cluster\_List[l]|-1\}, p \ne q$}
	%\IF{$dist(P,Q) = \min_{R \in Cluster\_List[level]}dist(P,R)$ \textbf{and} $dist(Q,P) = \min_{R \in Cluster\_List[level]}dist(Q,R)$}
	\IF{$dist(Cluster\_List[l][p],Cluster\_List[l][q])= min\_dist[p]$ \textbf{and} \\ $\;\;\;dist(Cluster\_List[l][p],Cluster\_List[l][q])= min\_dist[q]$}
		\STATE $\;\;\;\mathcal G[p,q]=1$
		\STATE $\;\;\;\mathcal G[q,p]=1$
	\ENDIF
\ENDFOR

%\vspace{2mm}

\COMMENT{Extract the connected components of $\mathcal G$}:

\STATE $CC = connected\_components(\mathcal G)$.
\FORALL{$cc \in CC$}
\STATE $new\_cluster = Cluster\_List[l][cc]$
\STATE $Cluster\_List[l+1].add(new\_cluster)$
\ENDFOR

\STATE $l = l + 1$

\ENDWHILE

\RETURN $Cluster\_List$
\end{algorithmic}
\end{algorithm}

In this algorithm, $Cluster\_List$ is used to store the clusters at different levels, such that $Cluster\_List[l][p]$ gives the $p^{th}$ cluster at the level $l$.  The variable $l$ indicates the current level while building the dendrogram. At the beginning, each individual object constitutes a separate cluster at level $0$. Next, the distance of each cluster at the current level $l$ (stored in $Cluster\_List[l]$) to its nearest neighbor is computed and stored in $min\_dist$. Function $dist(,)$ computes the inter-cluster distance between the two input clusters, according to a predefined criterion, e.g., \emph{single}, \emph{complete} and so on.
%\footnote{For the criteria that we study in the paper, the inter-cluster distances are \emph{symmetric}, i.e., $dist(C_i,C_j) = dist(C_j,C_i)$.}
Then, in graph $\mathcal G \in \{0,1\}^{|Cluster\_List[l]| \times |Cluster\_List[l]|}$ (whose nodes represent the cluster indices at $Cluster\_List[level]$) an edge is established if and only if the two respective clusters are nearest neighbors of each other (i.e., the linkage is \emph{reliable}). Note that a cluster might have several nearest neighbors, i.e., several clusters might have the same (smallest) distance from that. At the next step, the connected components of the graph $\mathcal G$ are extracted, where each of them represents a new cluster at a higher level. Thus, the clusters at the same connected components are combined to build a new single cluster at the higher level. This procedure (i.e., finding the nearest neighbors and combining them to build new higher-level clusters) continues until only one cluster is left at the highest level.

Notice that the several improvements developed for the standard strategy can be applied to this strategy as well.
The computational complexity of this  strategy is similar to the complexity of the standard variant. Both strategies establish in total $n-1$ linkages.  For this, they compute the inter-cluster distances (linkages) according to a priori fixed criterion, and for each selected linkage, they merge the respective clusters and update the new inter-cluster linkages.  Therefore, the  operations and the computations are similar, whereas the choice of specific linkages and/or the order might differ which can lead to different dendrograms. Selecting all reliable linkages, instead of the smallest one, may reduce the overall number of steps, but it might need more merges at each step. However, as mentioned, the total number of merges is the same for both strategies.

On the other hand, as mentioned, an important computational advantage of the reliable strategy is the possibility of \emph{early stopping}. It builds and develops several clusters simultaneously, whereas the standard approach develops fewer clusters at the same time. Thus, if we stop at early/intermediate steps, it is more likely that we obtain good representatives of different clusters. But, with the standard strategy, it could happen that only a few clusters are developed and the rest have not even been started yet. This might happen in particular when the clusters have diverse densities and shapes. Therefore, early-stopping, to reduce the computational time, can be more effective with \emph{reliable agglomerative clustering}. For example, the early clusters can be exposed to the user  to select only the interesting and relevant ones to develop further.

Algorithm~\ref{alg:strong_agglomerative} enables every object to \emph{potentially} participate in building the dendrogram from the beginning, depending on having a reliable linkage. In other words, establishing and selecting a linkage and therefore growth of a cluster depends only on the relation of an object/cluster to its neighbors, independent of the relations of the other object/clusters with each other. However, this is not the case for the standard agglomerative clustering.
Thus, if we stop the algorithm early, then, we will possibly have representatives of many clusters which correspond to the denser and more important (informative) parts. On the other hand, the outlier objects do not occur in the nearest neighborhood of many other clusters or objects. Thus, they join the other parts of the dendrogram only at the higher levels. Thereby, Algorithm~\ref{alg:strong_agglomerative} can be employed to provide a systematic way to separate structure from noise and outlier objects at different resolutions. The probability of object $i$ being an outlier is proportional to the level at which the object joins to the other objects/clusters,  i.e.,
\begin{equation}
\text{Pr}[i \in outlier] \propto l^*(i) \, ,  \nonumber
\end{equation}
where $l^*(i)$ specifies the level at which object $i$ joins to one of the other clusters/objects for the first time.
The higher $l^*(i)$ is, the larger the outlier probability is. We postpone the detail to future work.
% Thereby, in this framework, one can even analyze the outlier clusters (clusters with very few objects) in addition to outlier objects.

We may parametrize this strategy by a parameter such as $\alpha$ which specifies the ratio of the (smallest) \emph{reliable} linkages to be established at each step. A value close to zero  then corresponds to the standard variant, whereas $\alpha=1$ will be equal to the reliable strategy described in Algorithm\ref{alg:strong_agglomerative}. In this way, we can  provide even a richer family of alternative strategies for performing agglomerative clustering.

%\emph{Reliable} linkages are sometimes used for constructing $K$-nearest neighbor graphs for clustering, called mutual $K$-nearest neighbor graph~\cite{Brito1997,MaierHL09}. In this setting, each object is connected to its $K$ nearest neighbors, but a valid edge is established if and only if the objects at its two side select the other one. Then, the \emph{flat} clusters can be obtained by enumerating the connected components~\cite{Brito1997}. However, there are several fundamental differences between our approach and this method: The mutual $K$-nearest neighbor method establishes the edges between objects and not clusters.Then, most of the established edges occur inside small dense regions or small clusters. For this reason, as investigated in~\cite{Luxburg07,MaierHL09}, then the selected $K$ must be a large number, i.e. in the order of the number of objects $n$. However,  such a choice can easily lead to combining even well-separated clusters.  In particular, when the size of clusters varies or the clusters have elongated shapes, then even a fairly small $k$ can combine different clusters. Thereby, as suggested in ~\cite{Luxburg07,MaierHL09}, mutual $K$-nearest neighbor is  appropriate only for detecting significant cluster seeds. On the other hand, this method yields only flat clusters.

\section{Reliable Minimum Spanning Trees}

Minimum spanning trees (MSTs) are used in several applications such as transportation, computer and telecommunication networks \cite{Graham:1985},  image segmentation~\cite{Felzenszwalb:2004}, taxonomy learning~\cite{sneath1957dn09j} and power systems \cite{Mori1991}. It is known that the \emph{single} linkage method is equivalent to the Kruskal's algorithm~\cite{Kruskal1956} for computing a minimum spanning tree \cite{Gower1969}. Consistently, we study that Algorithm~\ref{alg:strong_agglomerative} with the \emph{single} criterion also yields a minimum spanning tree (Theorem~\ref{theorem:strong_single_mst}), which then its construction is adaptive w.r.t. the diverse density of the underlying data.

Before proving Theorem~\ref{theorem:strong_single_mst}, we first introduce some notations. Consider a forest (collection) of  trees $\left\{T_0, T_1, ...\right\}$. The distance (the edge weight) between the two trees $T_p$ and $T_q$ is computed according to the \emph{single} criterion, i.e.,
\begin{equation}
  \Delta_{pq} = \min_{i \in T_p} \min_{j \in T_q} \mathbf D_{ij}. \nonumber
\end{equation}
%Notice that $\Delta T_{pq}$ is symmetric, i.e. $\Delta T_{pq} = \Delta T_{qp}$.
The nearest tree from tree $T_p$, i.e. $T_p^*$, is obtained via
%\begin{equation}
  $T_p^* = \arg\min_{T_q}\Delta{pq},  q\ne p$.
%\end{equation}
Moreover, $e_p^*$ shows the edge corresponds to the nearest tree from $T_{p}$, i.e.,
\begin{equation}
  e_p^* = \arg\min_{e\in E} \Delta_{pq},  q\ne p \, ,  \nonumber
\end{equation}
where $E$ is the set of all current inter-tree edges.

\setcounter{theorem}{0}
\begin{theorem}
The dendrogram generated by Algorithm \ref{alg:strong_agglomerative} with the single linkage criterion computes a minimum spanning tree.
\label{theorem:strong_single_mst}
\end{theorem}
\begin{proof}[Proof sketch]
Consider a forest of trees $T_0, T_1, \cdots$. According to the connectivity condition of the final minimum spanning tree, every tree $T_p$ should be connected via an edge to the rest of the MST. This edge should be $e_p^*$, i.e., an edge (linkage) with minimal weight  among the edges of $T_p$, to keep the spanning tree minimal. Otherwise, if a larger edge is selected, then, the resultant spanning tree will have a larger total weight (i.e., a contradiction occurs). The linkage suggested by Algorithm~\ref{alg:strong_agglomerative} (with the \emph{single} criterion) satisfies this condition: The selected linkage is the smallest linkage connected to both $T_p$ and $T_p^*$ (the tree at the other side).

Hence, at the beginning, we consider each object as a separate tree, where all must belong to the final minimum spanning tree. Then, according to the aforementioned argument and based on induction, the edges selected at each step belong to the final MST. Thus, the final tree will be a minimum spanning tree.
\end{proof}

In this context, the \emph{generalized greedy algorithm}~\cite{Gabow1986} provides a general framework for computing minimum spanning trees, by showing that the edge  $e_{p*}$ is a consistent choice with a final minimum spanning tree.
%Given the collection of nonempty and disjoint trees $\{T_p\}$, assume we already know that all the trees are included in a minimum spanning tree to be constructed. Then,  one can show that the edge  $e_{p*}$ will be part of the minimum spanning tree~\cite{Gabow1986,Christofides:1975}.
%Consider a set of singleton subtrees $T_0, ..., T_{n-1}$, where each $T_i$ contains only the $i^{th}$ object (node).
Thereby, a greedy MST algorithm, at each step, i) picks $T_p$ and $T_p^*$, i.e., two candidate trees where at least one is the nearest neighbor of the other,
ii) combines them via the smallest edge $e_p^*$ to build a larger tree, and iii) removes the selected trees $T_p$ and $T_p^*$. The procedure continues until only a single tree with $n$ nodes remains, which is a MST.

Different algorithms, e.g. Kruskal's and Prim's \cite{Prim1957}, differ only in the way they pick the candidate trees at each step.
Kruskal's, at each step, picks a pair of trees that have a minimal distance among all pairs of trees. However, Prim's produces the MST via growing only one tree, say $T_0$, by  iteratively attaching a singleton tree which has minimal distance to that, until it contains all the singleton trees.

Algorithm~\ref{alg:strong_agglomerative} with the \emph{single} criterion yields an alternative viewpoint on the construction of  MSTs. According to the \emph{generalized greedy algorithm}, to combine two candidate trees, it is sufficient that one of them occurs in the nearest neighborhood of the other. However, Algorithm~\ref{alg:strong_agglomerative} requires that both trees mutually occur insides the nearest neighborhood of each other. As shown, e.g. in  Figures~\ref{fig:Context-Sensitive-Clust2} and~\ref{fig:Context-Sensitive-Clust3}, such a strategy yields a robust and adaptive minimum spanning tree. In summary,

%\begin{enumerate}
\textbf{I.} the standard agglomerative method, in a very strict way, selects only one \emph{reliable} linkage at each step, the one which has a minimal length (weight).

\textbf{II.} On the other hand, the \emph{generalized greedy algorithm} for MST construction allows one to select any edge which occurs inside the nearest neighbors of one of the trees, regardless of being \emph{reliable} or not (which might not be robust).

\textbf{III.} Algorithm~\ref{alg:strong_agglomerative}  follows an intermediate strategy. It suggests to select all the \emph{reliable} linkages (edges) which yields adaptation and flexibility (compared to the former approach) and robustness (compared to the latter approach).

\textbf{IV.} Parameterization of Algorithm~\ref{alg:strong_agglomerative} (by $\alpha$, as discussed before) can lead to an even larger family of different (\emph{reliable}) minimum spanning tree algorithms.

We note that the final MST obtained by Algorithm~\ref{alg:strong_agglomerative} could be the same as the Kruskal's MST. However, the order of  selecting the edges differs. Thus, in particular, if we stop early constructing the MST, then the available solution could be different.
%\end{enumerate}

\section{Experiments}
We experimentally evaluate the performance of the reliable agglomerative  strategy on a variety of real-world datasets and compare it against the standard approach. In these datasets, each object (i.e., document, image, etc) is represented by a vector according to the respective features. For the text documents, we use the tf-idf vectors.  We compute the pairwise dissimilarities between the objects according to squared Euclidean distance measure.

\paragraph{\textbf{Data}}
The first datasets are selected from the UCI data repository \cite{UCIDua:2019}.
\begin{enumerate}
\item \emph{Ecoli}: contains the information of $336$ protein localization sites in $7$ categories.
\item \emph{Hayes Roth}:  is related to a study on human subjects which contains $160$ instances and $3$ classes.
\item \emph{Iris}: contains the information of $150$ iris plants grouped in $3$ classes.
\item \emph{Lung Cancer}: includes $3$ types of $32$ instances of pathological lung cancer.
\item \emph{Perfume}: consists of odors of $20$ different perfumes (classes), where there are in total $560$ measurements.
\item \emph{Seeds}: includes $210$ measurements of geometrical properties of kernels belonging to different varieties of wheat.
\item \emph{Wine}: contains $178$ measurements of a chemical analysis of  different types of wines.
\end{enumerate}

We also use the three main subsets of 20-newsgroup data collection:

\begin{enumerate}
\item \emph{COMP}: a subset of  $1,955$ documents in five groups:  \emph{ `comp.graphics', `comp.windows.x', `comp.os.ms-windows.misc', `comp.sys.ibm.pc.hardware', `comp.sys.mac.hardware'}.
\item \emph{REC}: a subset of  $1,590$ documents in four groups related to race and sports: \emph{`rec.autos', `rec.motorcycles', `rec.sport.baseball', `rec.sport.hockey'}.
\item \emph{SCI}: a subset of  $1,579$ documents in four groups related to science: \emph{`sci.crypt', `sci.electronics', `sci.med', `sci.space'}.
\end{enumerate}

In addition, we investigate the performance of different strategies on real datasets collected by a document processing corporation.
%Xerox Corporation.
The original dataset (called \emph{Real I}) contains the vectors of $675$ scanned documents each represented in a $4,096$ dimensional space. This dataset contains $56$ clusters which several of them have only one or few documents. Then, by removing the clusters with only one or two documents, we obtain a new dataset, called \emph{Real II} ($634$ documents) .  Finally, we obtain \emph{Real III} by keeping the clusters that have at least $5$ documents ($592$ documents).
%Thus,
%\\
%%\begin{enumerate}
%11) \emph{Real I}: The original real-world dataset which contains vectors of $675$ documents, grouped in $56$ clusters.\\
%12) \emph{Real II}:  vectors of  $634$ documents in $34$ clusters.\\
%13) \emph{Real III}:  vectors of $592$ documents in $21$ clusters.
%%\end{enumerate}

\begin{table*}[th!]
\begin{changemargin}{-.5in}{-.5in}
\caption{Performance of standard (stnd) and reliable (rlbl) agglomerative strategies  w.r.t. Mutual Information. The reliable strategy often yields the improvment of the results.
}
\centering % used for centering table
\begin{tabular}{|| c || c c || c c || c c || c c || c c ||} % centered columns (4 columns)
\hline\hline %inserts double horizontal lines
 &\multicolumn{2}{c||}{single}&\multicolumn{2}{c||}{complete}&\multicolumn{2}{c||}{average}&\multicolumn{2}{c||}{centroid}&\multicolumn{2}{c||}{Ward}\\
 \hline
 dataset & stnd & rlbl & stnd & rlbl & stnd & rlbl & stnd & rlbl & stnd & rlbl \\ [0.5ex] % inserts table
%heading
\hline % inserts single horizontal line
\emph{Ecoli} & 0.0564 & 0.0564 & 0.6235 & 0.6235 & 0.5907 & \textbf{0.6812} & 0.0462 & 0.0383 & 0.5473 & 0.5445 \\  %[1ex] % [1ex] adds vertical space
\hline %inserts single line
\emph{Hayes Roth} & 0.0161 & \textbf{0.2336} & 0.0354 & \textbf{0.2338} & 0.1629 & \textbf{0.2338} & 0.0000 & 0.0030 & 0.0249 & 0.0808 \\  %[1ex] % [1ex] adds vertical space
\hline %inserts single line
%\emph{Iris} & 0.5821 & 0.5821 & 0.6963 & 0.6963 & 0.6301 & 0.6264 & 0.7934 & 0.7934 & 0.7578 & 0.7578 \\  %[1ex] % [1ex] adds vertical space
\emph{Iris} & 0.5821 & 0.5821 & 0.6963 & 0.6963 & 0.6301 & 0.6301 & \textbf{0.7934} & \textbf{0.7934} & 0.7578 & 0.7578 \\  %[1ex] % [1ex] adds vertical space
\hline %inserts single line
\emph{Lung Cancer} & 0.0149 & 0.0149 & 0.1537 & \textbf{0.2070} & 0.0239 & 0.1413 & 0.0000 & 0.0000 & 0.1766 & 0.1684 \\  %[1ex] % [1ex] adds vertical space
\hline %inserts single line
\emph{Perfume} & 0.7024 & 0.7024 & 0.7332 & 0.7332 & 0.7601 & 0.7595 & 0.7544 & 0.7664 & \textbf{0.8246} & \textbf{0.8246} \\  %[1ex] % [1ex] adds vertical space
\hline %inserts single line
\emph{Seeds} & 0.0283 & 0.0283 & 0.6029 & 0.6029 & 0.6083 & 0.7055 & 0.6034 & 0.6140 & \textbf{0.7243} & \textbf{0.7243} \\  %[1ex] % [1ex] adds vertical space
\hline %inserts single line
\emph{Wine} & 0.0237 & 0.0237 & \textbf{0.4307} & \textbf{0.4307} & 0.3223 & 0.3452 & 0.3251 & 0.3251 & 0.4097 & 0.4097 \\  %[1ex] % [1ex] adds vertical space
\hline %inserts single line
\emph{COMP} & 0.0604 & 0.0604 & 0.1459 & 0.1459 & 0.0453 & \textbf{0.1611} & 0.0312 & 0.0341 & 0.1021 & 0.1140   \\  %[1ex] % [1ex] adds vertical space   !!!!!!!! last two entried baraks!!!!!!!!!!!
\hline %inserts single line
\emph{REC} & 0.0228 & 0.0402 & 0.1793 & 0.1793 & 0.0330 & 0.2375 & 0.0161 & 0.0315 & \textbf{0.2574} & \textbf{0.2574}   \\  %[1ex] % [1ex] adds vertical space
\hline %inserts single line
\emph{SCI} &  0.0617 & 0.0617 & 0.0823 & 0.0823 & 0.0387 & 0.1557 & 0.0339 & 0.0651 & 0.1997 & \textbf{0.3042}  \\  %[1ex] % [1ex] adds vertical space
\hline %inserts single line
\emph{Real I} & 0.5782 & 0.5782 & 0.7114 & 0.7114 & 0.7813 & \textbf{0.8237} &0.0785 & 0.0670 & 0.5976 & 0.7546  \\  %[1ex] % [1ex] adds vertical space
\hline %inserts single line
\emph{Real II} & 0.5711 & 0.5711 & 0.7430 & 0.7430 & 0.7704 & 0.8130 & 0.0458 & 0.0268 & 0.6542 & \textbf{0.8274}  \\  %[1ex] % [1ex] adds vertical space
\hline %inserts single line
\emph{Real III} & 0.5389 & 0.5389 & 0.7581 & 0.7581 & 0.7209 & 0.7733 & 0.0132 & 0.0145 & 0.7156 & \textbf{0.8697} \\  %[1ex] % [1ex] adds vertical space
\hline %inserts single line
\end{tabular}
\label{table:results_mutual_info} % is used to refer this table in the text
\end{changemargin}
\end{table*}

\begin{table*}[th!]
\begin{changemargin}{-.5in}{-.5in}
\caption{Performance of standard (stnd) and reliable (rlbl) strategies with different criteria w.r.t. Rand score, where the reliable strategy usually gives superior results.
}
\centering % used for centering table
\begin{tabular}{|| c || c c || c c || c c || c c || c c ||} % centered columns (4 columns)
\hline\hline %inserts double horizontal lines
 &\multicolumn{2}{c||}{single}&\multicolumn{2}{c||}{complete}&\multicolumn{2}{c||}{average}&\multicolumn{2}{c||}{centroid}&\multicolumn{2}{c||}{Ward}\\
 \hline
 dataset & stnd & rlbl & stnd & rlbl & stnd & rlbl & stnd & rlbl & stnd & rlbl \\ [0.5ex] % inserts table
%heading
\hline % inserts single horizontal line
\emph{Ecoli} & 0.0386 & 0.0386 & 0.6908 & 0.6908 & 0.6974 & \textbf{0.7509} & 0.0297 & 0.0252 & 0.4686 & 0.3914 \\  %[1ex] % [1ex] adds vertical space
\hline %inserts single line
\emph{Hayes Roth} & 0.0185 & 0.2086 & 0.0327 & \textbf{0.2451} & 0.1620 & \textbf{0.2451} & 0.0000 & 0.0058 & 0.0496 & 0.1073 \\  %[1ex] % [1ex] adds vertical space
\hline %inserts single line
%\emph{Iris} & 0.5638 & 0.5638 & 0.6423 & 0.6423 & 0.5659 & 0.5676 & 0.7592 & 0.7592 & 0.7312 & 0.7312 \\  %[1ex] % [1ex] adds vertical space
\emph{Iris} & 0.5638 & 0.5638 & 0.6423 & 0.6423 & 0.5659 & 0.5659 & \textbf{0.7592} & \textbf{0.7592} & 0.7312 & 0.7312 \\  %[1ex] % [1ex] adds vertical space
\hline %inserts single line
\emph{Lung Cancer} & 0.0371 & 0.0371 & 0.2809 & \textbf{0.3533} & 0.1327 & 0.1170 & 0.0000 & 0.0000 & 0.3388 & 0.1698 \\  %[1ex] % [1ex] adds vertical space
\hline %inserts single line
\emph{Perfume} & 0.4667 & 0.4667 & 0.5096 & 0.5096 & 0.5651 & 0.5600 & 0.5600 & 0.5749 & \textbf{0.6590} & \textbf{0.6590} \\  %[1ex] % [1ex] adds vertical space
\hline %inserts single line
\emph{Seeds} & 0.0025 & 0.0025 & 0.5461 & 0.5461 & 0.5543 & \textbf{0.7320} & 0.5664 & 0.5626 & 0.7132 & 0.7132 \\ %[1ex] % [1ex] adds vertical space
\hline %inserts single line
\emph{Wine} & 0.0054 & 0.0054 & \textbf{0.3708} & \textbf{0.3708} & 0.2926 & 0.3204 & 0.3266 & 0.3266 & 0.3684 & 0.3684 \\ %[1ex] % [1ex] adds vertical space
\hline %inserts single line
\emph{COMP} & 0.0531 & 0.0531 & 0.1331 & 0.1331 & 0.0040 & \textbf{0.1459} & 0.0119 & 0.0138 & 0.0290 & 0.0296   \\  %[1ex] % [1ex] adds vertical space    !!!!!!!! last two entried baraks!!!!!!!!!!!
\hline %inserts single line
\emph{REC} & 0.0262 & 0.0742 & 0.0905 & 0.0905 & 0.0025 & 0.2266 & 0.0014 & 0.0052 & \textbf{0.2162} & \textbf{0.2162}   \\  %[1ex] % [1ex] adds vertical space
\hline %inserts single line
\emph{SCI} & 0.0884 & 0.0884 & 0.0108 & 0.0108 & 0.0034 & 0.0782 & 0.0493 & 0.0588 & 0.0908 & \textbf{0.1688}   \\  %[1ex] % [1ex] adds vertical space
\hline %inserts single line
\emph{Real I} & 0.4296 & 0.4296 & 0.4133 & 0.4133 & 0.4687 & \textbf{0.5699} & 0.0401 & 0.0403 & 0.2649 & 0.4969 \\  %[1ex] % [1ex] adds vertical space
\hline %inserts single line
\emph{Real II} & 0.4409 & 0.4409 & 0.4142 & 0.4142 & 0.5581 & \textbf{0.6685} & 0.0283 & 0.0198 & 0.3193 & 0.6176 \\  %[1ex] % [1ex] adds vertical space
\hline %inserts single line
\emph{Real III} & 0.4235 & 0.4235 & 0.4414 & 0.4414 & 0.6850 & 0.6443 & 0.0123 & 0.0151 & 0.4101 & \textbf{0.7042} \\  %[1ex] % [1ex] adds vertical space
\hline %inserts single line
\end{tabular}
\label{table:results_rand_score} % is used to refer this table in the text
\end{changemargin}
\end{table*}

\begin{table*}[htb!]
\begin{changemargin}{-.5in}{-.5in}
\caption{Performance of standard (stnd) and reliable (rlbl)  strategies w.r.t. V-measure.
The reliable strategy provides better results compared to the standard variant.
}
\centering % used for centering table
\begin{tabular}{|| c || c c || c c || c c || c c || c c ||} % centered columns (4 columns)
\hline\hline %inserts double horizontal lines
 &\multicolumn{2}{c||}{single}&\multicolumn{2}{c||}{complete}&\multicolumn{2}{c||}{average}&\multicolumn{2}{c||}{centroid}&\multicolumn{2}{c||}{Ward}\\
 \hline
 dataset & stnd & rlbl & stnd & rlbl & stnd & rlbl & stnd & rlbl & stnd & rlbl \\ [0.5ex] % inserts table
%heading
\hline % inserts single horizontal line
\emph{Ecoli} & 0.1355 & 0.1355 & 0.6789 & 0.6789 & 0.6683 & \textbf{0.7115} & 0.1008 & 0.0819 & 0.6123 & 0.5658 \\  %[1ex] % [1ex] adds vertical space
\hline %inserts single line
\emph{Hayes Roth} & 0.0579 & \textbf{0.3472} & 0.0556 & 0.3010 & 0.2164 & 0.3010 & 0.0000 & 0.0203 & 0.0412 & 0.0995 \\  %[1ex] % [1ex] adds vertical space
\hline %inserts single line
%\emph{Iris} & 0.7175 & 0.7175 & 0.7221 & 0.7221 & 0.7046 & 0.6837 & 0.8057 & 0.8057 & 0.7701 & 0.7701 \\  %[1ex] % [1ex] adds vertical space
\emph{Iris} & 0.7175 & 0.7175 & 0.7221 & 0.7221 & 0.7046 & 0.7046 & \textbf{0.8057} & \textbf{0.8057} & 0.7701 & 0.7701 \\  %[1ex] % [1ex] adds vertical space
\hline %inserts single line
\emph{Lung Cancer} & 0.0287 & 0.0287 & 0.1810 & \textbf{0.2303} & 0.0742 & 0.1743 & 0.0000 & 0.0000 & 0.2140 & 0.2030 \\  %[1ex] % [1ex] adds vertical space
\hline %inserts single line
\emph{Perfume} & 0.8117 & 0.8117 & 0.8251 & 0.8251 & 0.8437 & 0.8417 & 0.8380 & 0.8442 & \textbf{0.8796} & \textbf{0.8796} \\  %[1ex] % [1ex] adds vertical space
\hline %inserts single line
\emph{Seeds} & 0.0663 & 0.0663 & 0.6152 & 0.6152 & 0.6204 & 0.7094 & 0.6150 & 0.6260 & \textbf{0.7309} & \textbf{0.7309} \\ %[1ex] % [1ex] adds vertical space
\hline %inserts single line
\emph{Wine} & 0.0615 & 0.0615 & \textbf{0.4423} & \textbf{0.4423} & 0.4049 & 0.3920 & 0.4277 & 0.4277 & 0.4161 & 0.4161 \\ %[1ex] % [1ex] adds vertical space
\hline %inserts single line
\emph{COMP} & 0.0351 & 0.0351 & 0.1857 & 0.1857 & 0.0754 & 0.1922 & 0.0515 & 0.0558 & 0.1323 & 0.1468   \\  %[1ex] % [1ex] adds vertical space    !!!!!!!! last two entried baraks!!!!!!!!!!!
\hline %inserts single line
\emph{REC} &  0.0307 & 0.0614 & 0.2310 & 0.2310 & 0.0609 & 0.2737 & 0.0308 & 0.0569 & \textbf{0.3124} & \textbf{0.3124}  \\  %[1ex] % [1ex] adds vertical space
\hline %inserts single line
\emph{SCI} &  0.0518 & 0.0518 & 0.1337 & 0.1337 & 0.0714 & 0.2005 & 0.0270 & 0.0339 & 0.2546 & \textbf{0.3407}  \\  %[1ex] % [1ex] adds vertical space
\hline %inserts single line
\emph{Real I} & 0.7708 & 0.7708 & 0.8484 & 0.8484 & 0.8409 & \textbf{0.8714} & 0.2181 & 0.2016 & 0.7932 & 0.8421 \\  %[1ex] % [1ex] adds vertical space
\hline %inserts single line
\emph{Real II} & 0.7570 & 0.7570 & 0.8221 & 0.8221 & 0.8361 & \textbf{0.8725} & 0.1384 & 0.0925 & 0.8023 & 0.8614 \\  %[1ex] % [1ex] adds vertical space
\hline %inserts single line
\emph{Real III} & 0.7197 & 0.7197 & 0.8155 & 0.8155 & 0.8427 & 0.8408 & 0.0510 & 0.0594 & 0.8238 & \textbf{0.8951} \\  %[1ex] % [1ex] adds vertical space
\hline %inserts single line
\end{tabular}
\label{table:results_v_measure} % is used to refer this table in the text
\end{changemargin}
\end{table*}

\paragraph{\textbf{Evaluation}}
To investigate the quality of a dendrogram, \emph{cophenetic correlation} \cite{Sokal1962} is sometimes employed specially in biostatistics which measures the correlation between the dendrogram and the base dissimilarities between the objects. However, this evaluation measure has several issues, e.g. i) it considers only the direct distances and discards the manifolds or the elongated structures, and ii) its value is very sensitive to the way the inter-cluster distances are computed. For example, the two \emph{single} and \emph{Ward} criteria might lead to the same dendrograms, but their \emph{cophenetic correlation} could significantly differ, since they compute different types of distances between the clusters (which constitute the elements of a dendrogram).
However, in our experiments, we have access to the ground-truth, i.e. to the true labels of the objects. Thus, we may use early stopping up to  $K$ clusters or eliminate the last $K-1$ linkages from a dendrogram to produce $K$ clusters. There exist more involved methods to convert a dendrogram into a set of $K$ clusters, but they require fixing critical parameters in advance which finding their correct values is non-trivial in an unsupervised setting such as clustering. With both strategies, ties might occur when producing exactly $K$ clusters.
We tackle the problem in the same way as the common implementations do, e.g. we break the ties according to the order (index) of the clusters, where all other tricks are applicable to both approaches as well. Moreover, we observe such ties usually occur at the lower levels of the dendrogram, i.e., for a very large $K$. For a rather small $K$, which is the case in many clustering problems,  such ties are very rare. In real data  it does not often happen  that many real clusters are mutually the nearest neighbors of each other. Having multiple reliable linkages to establish occurs at the low or intermediate levels. Thus, at the higher level, where we remove the linkages, ties are not common.

We compare the true and the computed clusters according to three criteria:
\begin{enumerate}
\item Normalized Mutual Information \cite{Vinh:2010}, which measures the mutual information between the true and the estimated  solutions.
\item Normalized Rand score \cite{hubert1985comparing}, which computes the similarity between the two solutions.
%. \footnote{V-measure yields very consistent results.}
\item V-measure~\cite{RosenbergH07}, which obtains the harmonic mean of homogeneity and completeness.
\end{enumerate}

We compute the normalized variant of these measures, such that they yield zero for randomly estimated solutions and thereby any positive score indicates a (partially) consistent solution.

\paragraph{\textbf{Results}}
Tables \ref{table:results_mutual_info}, \ref{table:results_rand_score} and \ref{table:results_v_measure}   show the performance scores in order w.r.t. Normalized Mutual Information, Normalized Rand score and V-measure, where the best results for each dataset are bolded. We observe that on different datasets, the reliable agglomeration strategy always contributes to the best results. In most cases, it improves significantly the best results of the standard strategy, and in fewer cases it yields consistent results with that. Moreover, in few cases it could happen that for a non-optimal criterion, the reliable variant yields (slightly) worse results. However, such a criterion is not the best choice and the respective scores are not high compared to the alternatives. For example, on \emph{Real I} and \emph{Real II}   with the \emph{centroid} criterion, the standard strategy yields slightly better scores than the reliable strategy. However,  the \emph{centroid} criterion is not the best option and yields anyway very low scores. With a more appropriate criterion (e.g. \emph{average} and \emph{Ward}), the reliable strategy gives significantly higher scores. Note that the different evaluation measures are often consistent, but in some cases they might disagree. For example, on the \emph{Seeds} dataset, Normalized Mutual Information suggests the \emph{Ward} criterion as the best option, but Normalized Rand score selects  the \emph{average} criterion, although \emph{Ward} still yields high scores.
Finally, it is notable that we observe similar experimental runtimes for the two strategies, as they perform similar operations. For example, on the \emph{COMP} dataset and with the \emph{single} criterion, the runtimes of the standard and reliable strategies are $0.9135$ and $0.9208$ seconds. With the \emph{average} criterion, the runtimes  respectively are $0.5338$ and $0.5326$ seconds.

%\vspace{-5mm}
\section{Conclusion}
We investigated  an adaptive and density-consistent strategy for agglomerative clustering, wherein at each step we establish all the \emph{reliable} linkages, instead of establishing only the smallest one (consistent with the high-level method in \cite{Bruynooghe1977}). The two clusters connected by a \emph{reliable} linkage share similar properties, such that they select each other as a nearest neighbor. This strategy enables the dendrogram to be adaptive w.r.t. the diverse densities of different clusters and supports early stopping the clustering procedure.
%Moreover, it can provide separating structure from the oulier objects/clusters at different levels, which will be investigated in future work.
In the following,
%similar to the connection between \emph{single} linkage clustering and the Kruskal's algorithm
we studied how \emph{reliable agglomerative clustering} with the \emph{single} criterion can be used to produce a minimum spanning tree.  Finally, we performed  experiments on several real-world datasets to investigate  the  performance of the reliable agglomerative strategy.

\vspace{2mm}
\noindent\textbf{\emph{Acknowledgement.}} This work was partially done at Xerox Research Centre Europe (XRCE).

\bibliography{references}
\bibliographystyle{plain}

\end{document}